%% file: HaarPool_arXiv.tex
\documentclass{article}

\usepackage{microtype}
\usepackage{subfigure}
\usepackage{hyperref}

\usepackage[accepted]{icml2020}

\icmltitlerunning{Haar Graph Pooling}

\input{math_commands.tex}

\usepackage{threeparttable}
\usepackage{hyperref}
\usepackage{url}
\usepackage{tabularx,ragged2e}
\usepackage{booktabs}
\usepackage{graphicx}
\usepackage{algorithmic}
\usepackage{enumitem}
\usepackage{multirow}

\begin{document}

\twocolumn[
\icmltitle{Haar Graph Pooling}

\icmlsetsymbol{equal}{*}

\begin{icmlauthorlist}
\icmlauthor{Yu Guang Wang}{equal,unsw,mpi}
\icmlauthor{Ming Li}{equal,zjnu}
\icmlauthor{Zheng Ma}{equal,princeton}
\icmlauthor{Guido Mont\'{u}far}{equal,ucla,mpi}
\icmlauthor{Xiaosheng Zhuang}{cityu}
\icmlauthor{Yanan Fan}{unsw}
\end{icmlauthorlist}

\icmlaffiliation{unsw}{School of Mathematics and Statistics,
 University of New South Wales, Sydney, Australia}
\icmlaffiliation{zjnu}{Department of Educational Technology, Zhejiang Normal University, Jinhua, China}
\icmlaffiliation{princeton}{Department of Physics, Princeton University, New Jersey, USA}
\icmlaffiliation{ucla}{Department of Mathematics and Department of Statistics, University of California, Los Angeles}
\icmlaffiliation{mpi}{Max Planck Institute for Mathematics in the Sciences, Leipzig, Germany}
\icmlaffiliation{cityu}{Department of Mathematics, City University of Hong Kong, Hong Kong}

\icmlcorrespondingauthor{Ming Li}{mingli@zjnu.edu.cn}
\icmlcorrespondingauthor{Yu Guang Wang}{yuguang.wang@unsw.edu.au}
\icmlcorrespondingauthor{Guido Mont\'{u}far}{montufar@math.ucla.edu}

\icmlkeywords{Machine Learning, ICML}

\vskip 0.3in
]

\printAffiliationsAndNotice{\icmlEqualContribution}

\begin{abstract}
Deep Graph Neural Networks (GNNs) are useful models for graph classification and graph-based regression tasks. In these tasks, graph pooling is a critical ingredient by which GNNs adapt to input graphs of varying size and structure. We propose a new graph pooling operation based on compressive Haar transforms --- \emph{HaarPooling}.
HaarPooling implements a cascade of pooling operations; it is computed by following a sequence of clusterings of the input graph.
A HaarPooling layer transforms a given input graph to an output graph with a smaller node number and the same feature dimension; the compressive Haar transform filters out fine detail information in the Haar wavelet domain.
In this way, all the HaarPooling layers together synthesize the features of any given input graph into a feature vector of uniform size. Such transforms provide a sparse characterization of the data and preserve the structure information of the input graph.
GNNs implemented with standard graph convolution layers and HaarPooling layers achieve state of the art performance on diverse graph classification and regression problems.
\end{abstract}

\section{Introduction}
Graph Neural Networks (GNNs) have demonstrated excellent performance in node classification tasks and are very promising in graph classification and regression \citep{Bronstein_etal2017,Survey_Battaglia,Survey_ZhuWW,Survey_SunMS,Survey_ZhangCQ}.
In node classification, the input is a single graph with missing node labels that are to be predicted from the known node labels.
In this problem, GNNs with appropriate graph convolutions can be trained based on a single input graph, and achieve state-of-the-art performance \citep{DeBrVa2016,KiWe2017,MaLiWa2019}.
Different from node classification, graph classification is a task where the label of any given graph-structured sample is to be predicted based on a training set of labeled graph-structured samples.
This is similar to the image classification task tackled by traditional deep convolutional neural networks.
The significant difference is that here each input sample may have an arbitrary adjacency structure instead of the fixed, regular grids that are used in standard pixel images.
This raises two crucial challenges:
1)~How can GNNs exploit the graph structure information of the input data?
2)~How can GNNs handle input graphs with 
varying number of nodes and connectivity structures?

These problems have motivated the design of proper \emph{graph convolution} and \emph{graph pooling} to allow GNNs to capture the geometric information of each data sample \citep{zhang2018end,ying2018hierarchical,cangea2018towards, gao2019graph, knyazev2019understanding,ma2019graph,lee2019self}. Graph convolution plays an important role, especially in question 1).

The following is a widely utilized type of graph convolution layer, proposed by \cite{KiWe2017}:
\begin{equation}\label{eq:GCNConv}
    X^{\rm out} = \widehat{A}X^{\rm in} W.
\end{equation}
Here $\widehat{A}=\widetilde{D}^{-1/2}(A+I)\widetilde{D}^{-1/2}\in \R^{N\times N}$ is a normalized version of the adjacency matrix $A$ of the input graph, where $I$ is the identity matrix and $\widetilde{D}$ is the degree matrix for $A+I$. Further, $X^{\rm in}\in \R^{N\times d}$ is the array of $d$-dimensional features on the $N$ nodes of the graph, and $W\in\R^{d\times m}$ is the filter parameter matrix.
We call the graph convolution of \citet{KiWe2017} in Equation~(\ref{eq:GCNConv}) the \emph{GCN convolution}.

The graph convolution in Equation~(\ref{eq:GCNConv}) captures the structural information of the input in terms of $A$ (or $\widehat{A}$),
and $W$ transforms the feature dimension from $d$ to $m$. As the filter size $d\times m$ is independent of the graph size, it allows a fixed network architecture to process input graphs of varying sizes.
The GCN convolution preserves the number of nodes, and hence the output dimension of the network is not unique.
Graph pooling provides an effective way to overcome this obstacle.
Some existing approaches, EigenPooling  \citep{ma2019graph}, for example, incorporate both features and graph structure, which gives very good performance on graph classification.

In this paper, we propose a new graph pooling strategy based on a sparse Haar representation of the data, which we call \emph{Haar Graph Pooling}, or simply \emph{HaarPooling}. It is built on the \emph{Haar basis} \citep{WaZh2019,wang2020tight,LiHANet2019}, which is a localized wavelet-like basis on a graph.
We define HaarPooling by the compressive Haar transform of the graph data, which is a nonlinear transform that operates in the Haar wavelet domain by using the Haar basis. For an input data sample, which consists of a graph $\mathcal{G}$ and the feature on the nodes $X^{\rm in}\in \R^{N \times d}$, the compressive Haar transform is determined by the Haar basis vectors on $\mathcal{G}$ and maps $X^{\rm in}$ into a matrix $X^{\rm out}$ of dimension $d\times N_1$. The pooled feature $X^{\rm out}$ is in the Haar wavelet domain and extracts the coarser feature of the input.
Thus, HaarPooling can provide sparse representations of graph data that distill structural graph information.

The Haar basis and its compressive transform can be used to define cascading pooling layers, i.e., for each layer, we define an orthonormal Haar basis and its compressive Haar transform. Each HaarPooling layer pools the graph input from the previous layer to output with a smaller node number and the same feature dimension.
In this way, all the HaarPooling layers together synthesize the features of all graph input samples into feature vectors with the same size.
We then obtain an output of a fixed dimension, regardless of the size of the input.

The algorithm of HaarPooling is simple to implement as the Haar basis and the compressive Haar transforms can be computed by the explicit formula. The computation of HaarPooling is cheap, with nearly linear time complexity. HaarPooling can connect to any graph convolution. The GNNs with HaarPooling can handle multiple tasks.
Experiments in Section~\ref{sec:exp} demonstrate that the GNN with HaarPooling achieves state of the art performance on various graph classification and regression tasks.

\section{Related Work}\label{sec:relatedwork}
Graph pooling is a vital step when building a GNN model for graph classification and regression, as one needs a unified graph-level rather than node-level representation for graph-structured inputs of which size and topology are changing. The most direct pooling method, as provided by the graph convolutional layer \citep{duvenaud2015convolutional},  takes the global mean and sum of the features of the nodes as a simple graph-level representation. This pooling operation treats all the nodes equally and uses the global geometry of the graph.
ChebNet \citep{DeBrVa2016} used a graph coarsening procedure to build the pooling module, for which one needs a graph clustering algorithm to obtain subgraphs.  One drawback of this topology-based strategy is that it does not combine the node features in the pooling.
The global pooling method considers the information about node embeddings, which can achieve the entire graph representation.
As a general framework for graph classification and regression problems, MPNN \citep{Gilmer_etal2017} used the Set2Set method \citep{vinyals2015order} that would obtain a graph-level representation of the graph input samples.  The SortPool \citep{zhang2018end} proposed a method that could rank and select the nodes by sorting their feature representation and then feed them into a traditional 1-D convolutional or dense layer. These global pooling methods did not utilize the hierarchical structure of the graph, which may carry useful geometric information of data.

One notable recent argument is to build a differentiable and data-dependent pooling layer with learnable operations or parameters, which has brought a substantial improvement in graph classification tasks.
The DiffPool \citep{ying2018hierarchical} proposed a differentiable pooling layer that learns a cluster assignment matrix over the nodes relating to the output of a GNN model. One difficulty of DiffPool is its vast storage complexity, which is due to the computation of the soft clustering.
The TopKPooling \cite{cangea2018towards, gao2019graph, knyazev2019understanding} proposed a pooling method that samples a subset of essential nodes by manipulating a trainable projection vector. The Self-Attention Graph Pooling (SAGPool) \citep{lee2019self} proposed an analogous pooling that applied the GCN module to compute the node scores instead of the projection vector in the TopKPooling. These hierarchical pooling methods technically still employ mean/max pooling procedures to aggregate the feature representation of super-nodes. To preserve more edge information of the graph, EdgePool \cite{diehl2019towards} proposed to incorporate edge contraction. The StructPool \cite{Yuan2020StructPool} proposed a graph pooling that employed conditional random fields to represent the relation of different nodes.

The spectral-based pooling method suggests another design, which operates the graph pooling in the frequency domain, for example, the Fourier domain or the wavelet domain.
By its nature, the spectral-based approach can combine the graph structure and the node features. The Laplacian Pooling (LaPool) \citep{noutahi2019towards} proposed a pooling method that dynamically selected the centroid nodes and their corresponding follower nodes by using a graph Laplacian-based attention mechanism.  EigenPool \citep{ma2019graph} introduced a graph pooling that used the local graph Fourier transform to extract subgraph information. Its potential drawback lies in the inherent computing bottleneck for the Laplacian-based graph Fourier transform, given the high computational cost for the eigendecomposition of the graph Laplacian. Our HaarPooling is a spectral-based method that applies the Haar basis system in the node feature representation, as we now introduce.

\section{Haar Graph Pooling}
\label{sec:overview.haarpool}
In this section, we give an overview of the proposed HaarPooling framework. First we define the pooling architecture in terms of a coarse-grained chain, i.e., a sequence of graphs $(\gph_0, \gph_1, \ldots, \gph_K)$, where the nodes of the $(j+1)$th graph $\gph_{j+1}$ correspond to the clusters of nodes of the $j$th graph $\gph_{j}$ for each $j=0,\ldots,K-1$.
Based on the chain, we construct the Haar basis and the compressive Haar transform. The latter then defines the HaarPooling operation. Each layer in the chain determines which sets of nodes the network pools, and the compressive Haar transform synthesizes the information from the graph and node feature for pooling.

\paragraph{Chain of coarse-grained graphs for pooling}
Graph pooling amounts to defining a sequence of coarse-grained graphs.
In our chain, each graph is an induced graph that arises from grouping (clustering) certain subsets of nodes from the previous graph.
We use clustering algorithms to generate the groupings of nodes.
There are many good candidates, such as spectral clustering \citep{shi2000normalized}, $k$-means clustering \citep{Pakhira2014}, DBSCAN \citep{ester1996density}, OPTICS \citep{ankerst1999optics} and METIS \citep{KaKu1998}.
Any of these will work with HaarPooling.
Figure~\ref{fig:chain1} shows an example of a chain
with $3$ levels, for an input graph $\gph_0$.

\begin{figure}[ht]
\centering
\begin{minipage}{\columnwidth}
\centering	
\includegraphics[width=.9\textwidth]{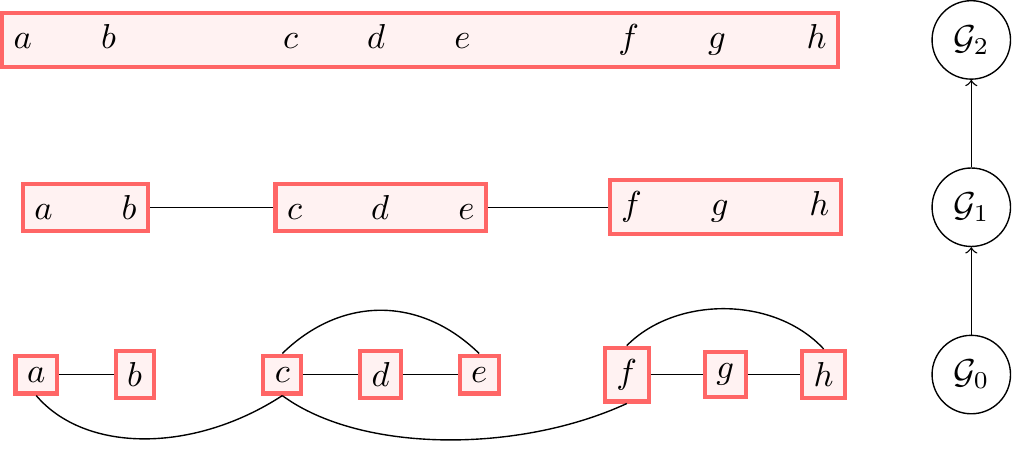}
\end{minipage}
\begin{minipage}{\columnwidth}
\centering
    \caption{A coarse-grained chain of graphs. The input has $8$ nodes; the second and top levels have $3$ and single nodes.}
    \label{fig:chain1}
\end{minipage}
\end{figure}

\paragraph{Compressive Haar transforms on chain}
For each layer of the chain, we will have a feature representation.
We define these in terms of the Haar basis.
The Haar basis represents graph-structured data by low- and high-frequency Haar coefficients in the frequency domain. The low-frequency coefficients contain the coarse information of the original data, while the high-frequency coefficients contain the fine details. In HaarPooing, the data is pooled (or compressed) by discarding the fine detail information.

The Haar basis can be compressed in each layer. Consider a chain. The two subsequent graphs have $N_{j+1}$ and $N_j$ nodes, $N_{j+1}<N_j$.
We select $N_j$ elements from the $(j+1)$th layer for the $j$th layer, each of which is a vector of size $N_{j+1}$. These $N_j$ vectors form a matrix $\Phi_{j}$ of size $N_{j+1}\times N_j$.
We call $\Phi_{j}$ the \emph{compressive Haar basis matrix} for this particular $j$th layer. This then defines the \emph{compressive Haar transform} $\Phi_{j}^{T}X^{\rm in}$ for feature $X^{\rm in}$ of size $N_{j}\times d$.

\paragraph{Computational strategy of HaarPooling}
By compressive Haar transform, we can define the HaarPooling.
\begin{definition}[HaarPooling]\label{defn:haarpool}
    The HaarPooling for a graph neural network with $K$ pooling layers is defined as
    \begin{equation*}
        X_j^{\rm out} = \Phi_j^T
        X_j^{\rm in},\quad j=0,1,\dots,K-1,
    \end{equation*}
where $\Phi_j$ is the $N_{j}\times N_{j+1}$ compressive Haar basis matrix for the $j$th layer, $X_j^{\rm in}\in \R^{N_{j}\times d_j}$ is the input feature array, and $X_j^{\rm out}\in \R^{N_{j+1}\times d_j}$ is the output feature array,
for some $N_j>N_{j+1}$, $j=0,1,\dots,K-1$, and $N_K=1$. For each $j$, the corresponding layer is called the $j$th HaarPooling layer.
More explicitly, we also write $\Phi_j$ as $\Phi^{(j)}_{N_j\times N_{j+1}}$.
\end{definition}

HaarPooling has the following fundamental properties.

\begin{itemize}[leftmargin=4mm]
\item First, the HaarPooling is a hierarchically structured algorithm. The coarse-grained chain determines the hierarchical relation in different HaarPooling layers. The node number of each HaarPooling layer is equal to the number of nodes of the subgraph of the corresponding layer of the chain. As the top-level of the chain can have one node, the HaarPooling finally reduces the number of nodes to one, thus producing a fixed dimensional output in the last HaarPooling layer.

    \item The HaarPooling uses the sparse Haar representation on chain structure. In each HaarPooling layer, the representation then combines the features of input $X_j^{\rm in}$ with the geometric information of the graphs of the $j$th and $(j+1)$th layers of the chain.

    \item By the property of the Haar basis, the HaarPooling only drops the high-frequency information of the input data. The $X_{j}^{\rm out}$ mirrors the low-frequency information in the Haar wavelet representation of $X_j^{\rm in}$. Thus, HaarPooling preserves the essential information of the graph input, and the network has small information loss in pooling.
\end{itemize}

\begin{figure*}
    \centering
    \begin{minipage}{0.47\textwidth}
    \centering
     \includegraphics[height=5cm]{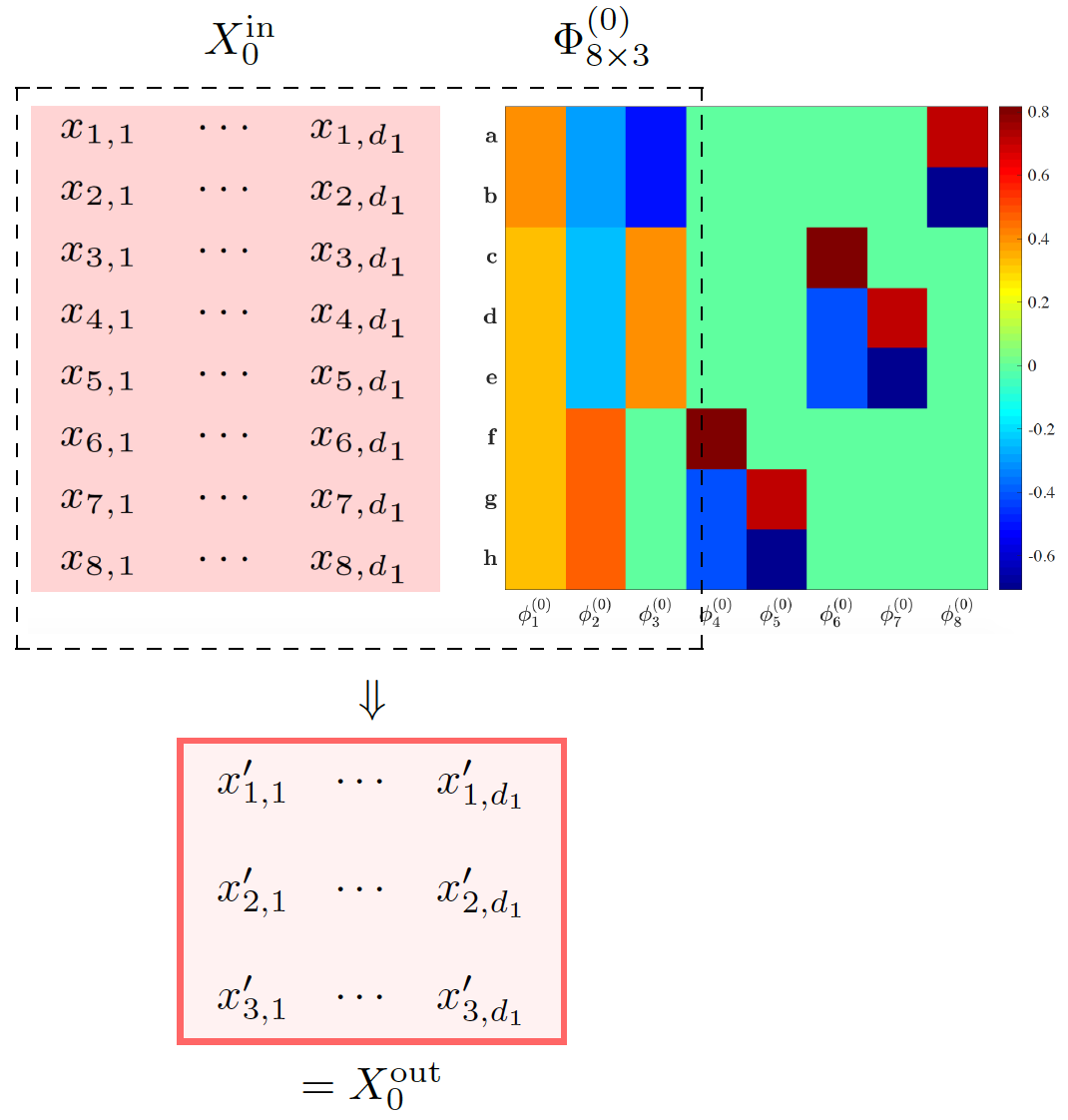}\\
     {\small (a) First HaarPooling Layer for $\gph_0\to\gph_1$.}
    \end{minipage}
    \hspace{.1cm}
    \begin{minipage}{0.47\textwidth}
    \centering
     \includegraphics[width=0.8\columnwidth,height=5cm]{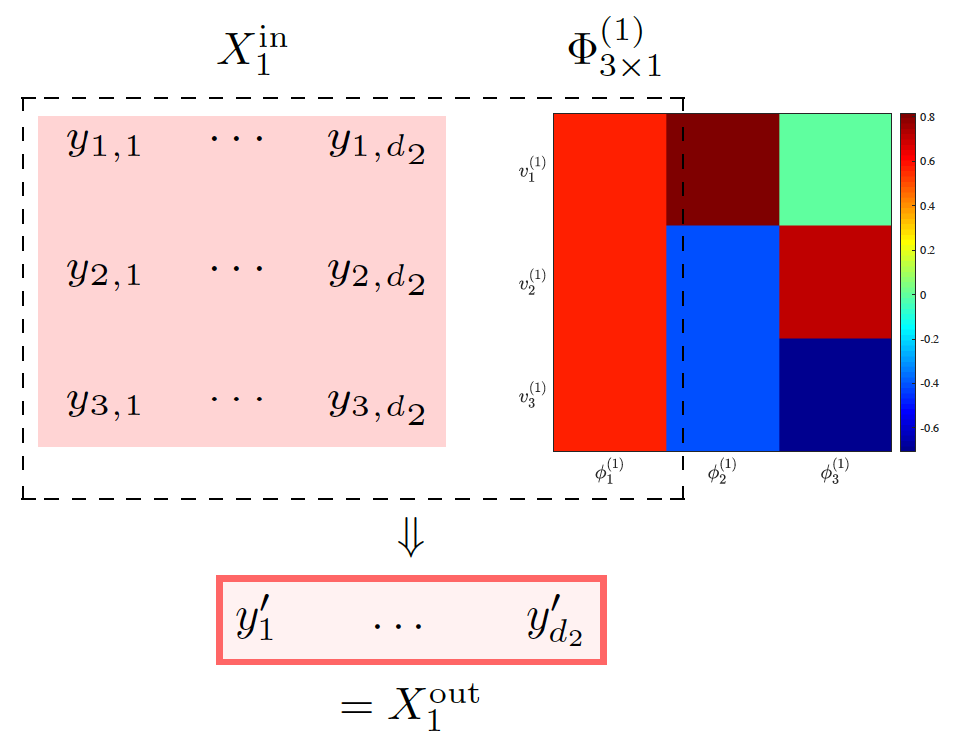}\\
     {\small (b) Second HaarPooling Layer for $\gph_1\to\gph_2$.}
    \end{minipage}
    \caption{Computational strategy of HaarPooling. We use the chain in Figure~\ref{fig:chain1}, and then the network has two HaarPooling layers: $\gph_0\to\gph_1$ and $\gph_1\to \gph_2$. The input of each layer is pooled by the compressive Haar transform for that layer: in the first layer, the input $X^{\rm in}_1=(x_{i,j})\in \R^{8\times d_1}$ is transformed by the compressive Haar basis matrix $\Phi_{8\times3}^{(0)}$ of size $8\times3$ formed by the first three column vectors of the original Haar basis, and the output is a feature array of size $3\times d_1$; in the second layer, $X^{\rm in}_2=(y_{i,j})\in\R^{3\times d_2}$ is transformed by the first column vector $\Phi_{3\times1}^{(1)}$ and the output is a feature vector of size $1\times d_2$. In the plots of the Haar basis matrix, the colors indicate the value of the entries of the Haar basis matrix.}
    \label{fig:comput.HaarPool}
\end{figure*}

\paragraph{Example}
Figure~\ref{fig:comput.HaarPool} shows the computational details of the HaarPooling associated with the chain from Figure~\ref{fig:chain1}.
There are two HaarPooling layers.
In the first layer, the input $X_1^{\rm in}$ of size $8\times d_1$ is transformed by the compressive Haar basis matrix $\Phi_{8\times3}^{(0)}$ which consists of the first three column vectors of the full Haar basis $\Phi_{8\times8}^{(0)}$ in (a), and the output is a $3\times d_1$ matrix $X_1^{\rm out}$.
In the second layer, the input $X_2^{\rm in}$ of size $3\times d_2$ (usually $X_1^{\rm out}$ followed by convolution) is transformed by the compressive Haar matrix $\Phi_{3\times 1}^{(1)}$, which is the first column vector of the full Haar basis matrix $\Phi_{3\times3}^{(1)}$ in (b).
By the construction of the Haar basis in relation to the chain (see Section~\ref{sec:math.haarpool}),
each of the first three column vectors $\eigfm[1]^{(0)}, \eigfm[2]^{(0)}$ and $\eigfm[3]^{(0)}$ of $\Phi^{(0)}_{8\times3}$ has only up to three different values.
This bound is precisely the number of nodes of $\gph_1$.
For each column of $\eigfm[\ell]^{(0)}$, all nodes with the same parent take the same value.
Similarly, the $3\times 1$ vector $\eigfm[1]^{(1)}$ is constant.
This example shows that the HaarPooling amalgamates the node feature by adding the same weight to the nodes that are in the same cluster of the coarser layer, and in this way, pools the feature using the graph clustering information.

\section{Compressive Haar Transforms}
\label{sec:math.haarpool}
\paragraph{Chain of graphs by clustering}
For a graph $\gph=(V,E,w)$, where $V,E,w$ are the vertices, edges, and weights on edges, a graph $\gph^{\rm cg}=(V^{\rm cg},E^{\rm cg},w^{\rm cg})$ is a \emph{coarse-grained graph} of $\gph$ if $|V^{\rm cg}|\leq |V|$ and each node of $\gph$ has only one parent node in $\gph^{\rm cg}$ associated with it. Each node of $\gph^{\rm cg}$ is called a \emph{cluster} of $\gph$.
For integers $J>0$, a \emph{coarse-grained chain} for $\gph$ is a sequence of graphs $\gph_{0\to J}:=(\gph_{0},\gph_1,\dots,\gph_{J})$ with $\gph_0=\gph$ and such that $\gph_{j+1}$ is a coarse-grained graph of $\gph_{j}$ for each $j=0,1,\dots,J-1$, and $\gph_{J}$ has only one node. Here, we call the graph $\gph_{J}$ the \emph{top level} or the \emph{coarsest level} and $\gph_{0}$ the \emph{bottom level} or the \emph{finest level}.
The chain $\gph_{0\to J}$ hierarchically coarsens graph $\gph$. We use the notation $J+1$ for the number of layers of the chain, to distinguish it from the number $K$ of layers for pooling.
For details about graphs and chains, we refer the reader to the examples by \cite{ChGr1997,HaVaGr2011,ChFiMh2015,ChMhZh2018,WaZh2019,wang2020tight}.

\citet{Haar1910} first introduced Haar basis on the real axis. It is a particular example of the more general Daubechies wavelets \citep{Daubechies1992}. Haar basis was later constructed on graphs by \cite{BeNiSi2006}, and also \cite{ChFiMh2015,WaZh2019,wang2020tight,LiHANet2019}.

\paragraph{Construction of Haar basis} The Haar bases $\{\phi_{\ell}^{(j)}\}_{\ell=1}^{N_j}$, $j=0,\dots,J$, is a sequence of collections of vectors. Each Haar basis is associated with a single layer of the chain $\gph_{0\to J}$ of a graph $\gph$. For $j=0,\dots,J$,
we let the matrix $\widetilde{\Phi}_j=(\phi_1^{(j)},\dots,\phi_{N_j}^{(j)})\in \R^{N_j\times N_j}$ and call the matrix $\widetilde{\Phi}_j$ \emph{Haar transform matrix} for the $j$th layer.
In the following, we detail the construction of the Haar basis based on the coarse-grained chain of a graph, as discussed in \cite{WaZh2019,wang2020tight,LiHANet2019}. We attach the algorithmic pseudo-codes for generating the Haar basis on the graph in the supplementary material.

\textbf{Step~1.} Let $\gph^{\rm cg}=(V^{\rm cg},E^{\rm cg},w^{\rm cg})$ be a coarse-grained graph of $\gph=(V,E,w)$ with $N^{\rm cg}:=|V^{\rm cg}|$. Here we use the sub-index ``cg'' to indicate the symbol is for the coarse-grained graph. Each vertex $v^{\rm cg}\in V^{\rm cg}$ is a cluster $v^{\rm cg}=\{v\in V\,|\,  v\mbox{ has parent } v^{\rm cg}\}$ of $\gph$. Order $V^{\rm cg}$, e.g., by degrees of vertices or weights of vertices, as $V^{\rm cg}=\{v^{\rm cg}_1,\ldots,v^{\rm cg}_{N^{\rm cg}}\}$. We define $N^{\rm cg}$ vectors $\eigfm^{\rm cg}$ on $\gph^{\rm cg}$ by
\begin{equation}\label{eq:haargc1}
\eigfm[1]^{\rm cg}(v^{\rm cg})  :=\frac{1}{\sqrt{N^{\rm cg}}},
\quad v^{\rm cg}\in V^{\rm cg},
\end{equation}
and for $\ell=2,\ldots,N^{\rm cg}$,
\begin{equation}\label{eq:haargc2}
\eigfm[\ell]^{\rm cg}:=\sqrt{\frac{N^{\rm cg}-\ell+1}{N^{\rm cg}-\ell+2}}\left(\chi^{\rm cg}_{\ell-1}-\frac{\sum_{j=\ell}^{N^{\rm cg}}\chi_j^{\rm cg}}{{N^{\rm cg}-\ell+1}}\right),
\end{equation}
where  $\chi_j^{\rm cg}$ is the indicator function for the $j$th vertex $v_j^{\rm cg}\in V^{\rm cg}$ on $\gph$ given by
\[
\chi_j^{\rm cg}(v^{\rm cg}) :=
\begin{cases}
1, & v^{\rm cg} = v_j^{\rm cg},\\
0, & v^{\rm cg}\in V^{\rm cg}\backslash \{v_j^{\rm cg}\}.
\end{cases}
\]
Then, the $\{\phi_{\ell}^{\rm cg}\}_{\ell=1}^{N^{\rm cg}}$ forms an orthonormal basis for $l_2(\gph^{\rm cg})$.
Each $v\in V$ belongs to exactly one cluster $v^{\rm cg}\in V^{\rm cg}$. In view of this,
for each $\ell=1,\dots,N^{\rm cg}$, we can extend the vector $\phi_\ell^{\rm cg}$ on $\gph^{\rm cg}$ to a vector $\eigfm[\ell,1]$ on $\gph$ by
\begin{equation*}
\eigfm[\ell,1](v):=
\frac{\eigfm^{\rm cg}(v^{\rm cg})}{\sqrt{|v^{\rm cg}|}}, \quad  v\in  v^{\rm cg},
\end{equation*}
here $|v^{\rm cg}|:=k_\ell$ is the size of the cluster $v^{\rm cg}$, i.e., the number of vertices in $\gph$ whose common parent is $v^{\rm cg}$.  We  order the cluster $v_\ell^{\rm cg}$,  e.g., by degrees of vertices, as
\[
v_\ell^{\rm cg} = \{v_{\ell,1},\ldots,v_{\ell,k_\ell}\}\subseteq V.
\]
For $k=2,\ldots,k_\ell$, similar to Equation~(\ref{eq:haargc2}), define
\begin{equation*}
\eigfm[\ell,k] =
 \sqrt{\frac{k_\ell-k+1}{k_\ell-k+2}}\left(\chi_{\ell,k-1}-\frac{\sum_{j=k}^{k_\ell}\chi_{\ell,j}}{k_\ell-k+1}\right),
\end{equation*}
where for $j=1,\dots,k_{\ell}$, $\chi_{\ell,j}$  is given by
\[
\chi_{\ell,j}(v) :=
\begin{cases}
1, & v = v_{\ell,j},\\
0, & v \in V\backslash \{v_{\ell,j}\}.
\end{cases}
\]
Then, the resulting $\{\phi_{\ell,k}: \ell=1,\dots,N^{\rm cg}, k=1,\dots,k_{\ell}\}$ is an orthonormal basis for $l_2(\gph)$.

\textbf{Step~2.} Let $\gph_{0\to J}$ be a coarse-grained chain for the graph $\gph$. An orthonormal basis $\{\phi_{\ell}^{(J)}\}_{\ell=1}^{N_J}$ for $l_2(\gph_{J})$ is generated using Equations~(\ref{eq:haargc1}) and (\ref{eq:haargc2}). We then repeatedly use Step~1: for $j=0,\dots,J-1$, generate an orthonormal basis $\{\phi_{\ell}^{(j)}\}_{\ell=1}^{N_j}$ for $l_2(\gph_j)$ from the orthonormal basis $\{\phi_{\ell}^{(j+1)}\}_{\ell=1}^{N_{j+1}}$ for the coarse-grained graph $\gph_{j+1}$ that was derived in the previous steps. We call the sequence $\{\phi_{\ell}:=\phi_{\ell}^{(0)}\}_{\ell=1}^{N_0}$ of vectors at the finest level,  the \emph{Haar global orthonormal basis}, or simply the \emph{Haar basis}, for $\gph$ associated with the chain $\gph_{0\to J}$. The orthonormal basis $\{\phi_{\ell}^{(j)}\}_{\ell=1}^{N_j}$ for $l_2(\gph_j)$, $j=1,\dots,J$ is called the \emph{Haar basis} for the $j$th layer.

\paragraph{Compressive Haar basis}
Suppose we have constructed the (full) Haar basis $\{\phi_{\ell}^{(j)}\}_{\ell=0}^{N_j}$ for each layer $\gph_j$ of the chain $\gph_{0\to K}$. The \emph{compressive Haar basis} for layer $j$ is $\{\phi_{\ell}^{(j)}\}_{\ell=0}^{N_{j+1}}$. We will use the transforms of this basis to define HaarPooling.

\paragraph{Orthogonality}
For each level $j=0,\dots,J$, the sequence $\{\phi_{\ell}^{(j)}\}_{\ell=1}^{N_j}$, with $N_j:=|V_j|$, is an orthonormal basis for the space $l_2(\gph_j)$ of square-summable sequences on the graph $\gph_j$, so that $(\phi_{\ell}^{(j)})^T\phi_{\ell'}^{(j)}=\delta_{\ell,\ell'}$.
For each $j$, $\{\phi_{\ell}^{(j)}\}_{\ell=1}^{N_j}$ is the Haar basis system for the chain $\gph_{j\to J}$.

\paragraph{Locality}
Let $\gph_{0 \to J}$ be a coarse-grained chain for $\gph$. If each parent of level $\gph_j$, $j=1,\dots,J$, contains at least two children, the number of different scalar values of the components of the Haar basis vector $\phi_{\ell}^{(j)}$, $\ell=1,\dots,N_j$, is bounded by a constant independent of $j$.

In Figure~\ref{fig:comput.HaarPool}, the Haar basis is generated based on the coarse-grained chain $\gph_{0\to 2}:=(\gph_0,\gph_1,\gph_2)$, where $\gph_0,\gph_1,\gph_2$ are graphs with $8, 3, 1$ nodes. The two colorful matrices show the two Haar bases for the layers $0$ and $1$ in the chain $\gph_{0\to 2}$. There are in total $8$ vectors of the Haar basis for $\gph_0$ each with length $8$, and $3$ vectors of the Haar basis for $\gph_1$ each with length $3$.
Haar basis matrix for each level of the chain has up to $3$ different values in each column, as indicated by colors in each matrix.
For $j=0,1$, each node of $\gph_{j}$ is a cluster of nodes in $\gph_{j+1}$. Each column of the matrix is a member of the Haar basis on the individual layer of the chain. The first three column vectors of $\widetilde{\Phi}_1$ can be reduced to an orthonormal basis of $\gph_1$ and the first column vector of $\gph_1$ to the constant basis for $\gph_2$. This connection ensures that the compressive Haar transforms for HaarPooling is also computationally feasible.

\paragraph{Adjoint and forward Haar transforms}
We utilize adjoint and forward Haar transforms to compute HaarPooling. Due to the sparsity of the Haar basis matrix, the transforms are computationally feasible.
The \emph{adjoint Haar transform} for the signal $f$ on $\gph_j$ is
\begin{equation}\label{eq:adft}
    (\widetilde{\Phi}_j)^T f = \left(\sum_{v\in V}\phi_1^{(j)}(v)f(v),\dots,\sum_{v\in V}\phi_{N_j}^{(j)}(v)f(v)\right)\in\R^{N_j},
\end{equation}
and the \emph{forward Haar transform} for (coefficients) vector $c:=(c_1,\dots,c_{N_j})\in \R^{N_j}$ is
\begin{equation}\label{eq:dft}
    (\widetilde{\Phi}_j c)(v) = \sum_{\ell=1}^{N_j}\phi_{\ell}^{(j)}(v) c_{\ell},\quad v\in V_j.
\end{equation}
We call the components of $(\widetilde{\Phi}_j)^T f$ the \emph{Haar (wavelet) coefficients} for $f$.
The adjoint Haar transform represents the signal in the Haar wavelet domain by computing the Haar coefficients for graph signal, and the forward transform sends back the Haar coefficients to the time domain.
Here, the adjoint and forward Haar transforms can be extended to a feature data with size $N_j\times d_j$ by replacing the column vector $f$ with the feature array.
\begin{proposition}\label{prop:signalrecover} The adjoint and forward Haar Transforms are invertible in that for $j=0,\dots,J$ and vector $f$ on graph $\gph_j$,
\begin{equation*}
    f = \widetilde{\Phi}_j (\widetilde{\Phi}_j)^T f.
\end{equation*}
\end{proposition}
Proposition~\ref{prop:signalrecover} shows that the forward Haar transform can recover the graph signal $f$ from the adjoint Haar transform $(\widetilde{\Phi}_j)^T f$, which means that adjoint and forward Haar transforms have zero-loss in graph signal transmission.

\paragraph{Compressive Haar transforms}
Now for a graph neural network, suppose we want to use $K$ pooling layers for $K\geq1$. We associate the chain $\gph_{0\to K}$ of an input graph with the pooling by linking the $j$th layer of pooling with the $j$th layer of the chain. Then, we can use the Haar basis system on the chain to define the pooling operation.
By the property of Haar basis, in the Haar transforms for layer $j$, $0\leq j\leq K-1$, of the $N_{j}$ Haar coefficients, the first $N_{j+1}$ coefficients are the low-frequency coefficients, which reflect the approximation to the original data, and the remaining $(N_{j}-N_{j+1})$ coefficients are in high frequency, which contains fine details of the Haar wavelet decomposition.
To define pooling, we remove the high-frequency coefficients in the Haar wavelet representation and then obtain the \emph{compressive Haar transforms} for the feature $X_j^{\rm in}$ at layers $j=0,\dots,K-1$, which then gives the HaarPooling in Definition~\ref{defn:haarpool}.

As shown in the following formula, the compressive Haar transform incorporates the neighborhood information of the graph signal as compared to the full Haar transform. Thus, the HaarPooling can take the average information of the data $f$ over nodes in the same cluster.
\begin{equation}\label{eq:norm.pooling}
\begin{aligned}
    \left\|\Phi_{j}^T X_j^{\rm in}\right\|^2
    &=\sum_{p\in\gph_{j+1}}\frac{1}{|\parents(v)|}
    \Bigl|\sum_{p=\parents(v)}X_j^{\rm in}(v)\Bigr|^2\\
   \left\|\widetilde{\Phi}_{j}^T X_j^{\rm in}\right\|^2
   &=\sum_{p\in\gph_{j+1}}
    \sum_{p=\parents(v)}\Bigl|X_j^{\rm in}(v)\Bigr|^2,
\end{aligned}
\end{equation}
where $\widetilde{\Phi}_{j}$ is the full Haar basis matrix at the $j$th layer and $|\parents_{\gph}(v)|$ is the number of nodes in the cluster which the node $v$ lies in. Here, $1/\sqrt{|\parents_{\gph}(v)|}$ can be taken out of summation as $\parents(v)$ is in fact a set of nodes. We show the derivation of formula in Equation~(\ref{eq:norm.pooling}) in the supplementary.

In HaarPooling, the compression or pooling occurs in the Haar wavelet domain. It transforms the features on the nodes to the Haar wavelet domain. It then discards the high-frequency coefficients in the sparse Haar wavelet representation.
See Figure~\ref{fig:comput.HaarPool} for a two-layer HaarPooling example.

\section{Experiments}
\label{sec:exp}
In this section, we present the test results of HaarPooling on various datasets in graph classification and regression tasks. We show a performance comparison of the HaarPooling with existing graph pooling methods.
All the experiments use PyTorch Geometric \citep{fey2019fast} and were run in Google Cloud using 4 Nvidia Telsa T4 with 2560 CUDA cores, compute 7.5, 16GB GDDR6 VRAM.

\subsection{HaarPooling on Classification Benchmarks}
\paragraph{Datasets and baseline methods}
To verify whether the proposed framework can hierarchically learn good graph representations for classification, we evaluate \emph{HaarPooling} on five widely used benchmark datasets for graph classification~\citep{KKMMN2016}, including one
protein graph dataset
{\bf PROTEINS}~\citep{borgwardt2005protein,dobson2003distinguishing};
two mutagen datasets \textbf{MUTAG}~\citep{Debnath_etal1991,kriege2012subgraph} and {\bf MUTAGEN}~\citep{riesen2008iam,kazius2005derivation} (full name Mutagenicity); and two datasets that consist of chemical compounds screened for activity against non-small cell lung cancer and ovarian cancer cell lines, {\bf NCI1} and {\bf NCI109}~\citep{wale2008comparison}.
We include datasets from different domains, samples, and graph sizes to give a comprehensive understanding of how the HaarPooling performs with datasets in various scenarios. Table~\ref{tab:statistics} summarizes some statistical information of the datasets: each dataset containing graphs with different sizes and structures, the number of data samples ranges from 188 to 4,337, the average number of nodes is from 17.93 to 39.06, and the average number of edges is from 19.79 to 72.82.
\begin{table*}[t]
\caption{Summary statistics of the graph classification datasets.}\label{tab:statistics}
\begin{center}
\begin{small}
\begin{tabularx}{380pt}{l *7{>{\Centering}X}}
\toprule
{\bf Dataset}  & MUTAG & PROTEINS& NCI1 & NCI109 & MUTAGEN \\
\midrule
max \#nodes  &28 &620  &111 &  111&417\\
min \#nodes  &10 &4  & 3& 4&4 \\
avg \#nodes &17.93 & 39.06  & 29.87 &29.68&30.32\\
avg \#edges &19.79 &72.82 &32.30 &32.13 &30.77\\
\#graphs    &188 & 1,113  & 4,110&4,127& 4,337\\
\#classes & 2& 2 &2 &2& 2\\
\bottomrule
\end{tabularx}
\end{small}
\end{center}
\end{table*}
We compare \textbf{HaarPool} with \textbf{SortPool} \citep{zhang2018end}, \textbf{DiffPool} \citep{ying2018hierarchical}, \textbf{gPool} \citep{gao2019graph}, \textbf{SAGPool} \citep{lee2019self}, \textbf{EigenPool} \citep{ma2019graph}, \textbf{CSM} \citep{kriege2012subgraph} and \textbf{GIN} \citep{GIN} on the above datasets.

\paragraph{Training}
In the experiment, we use a GNN with at most $3$ GCN \citep{KiWe2017} convolutional layers plus one HaarPooling layer, followed by three fully connected layers. The hyperparameters of the network are adjusted case by case. We use spectral clustering to generate a chain with the number of layers given. Spectral clustering, which exploits the eigenvalues of the graph Laplacian, has proved excellent performance in coarsening a variety of data patterns and can handle isolated nodes.

We apply random shuffling for the dataset. We split the whole dataset into the training, validation, and test sets with percentages 80\%, 10\%, and 10\%, respectively. We use the Adam optimizer \citep{kingma2014adam}, early stopping criterion, and patience, and give the specific values in the supplementary.
Here, the early stopping criterion was that the validation loss does not improve for 50 epochs, with a maximum of 150 epochs, as suggested by \citet{shchur2018pitfalls}.

The architecture of GNN is identified by the layer type and the number of hidden nodes at each layer. For example, we denote 3GC256-HP-2FC256-FC128 to represent a GNN architecture with 3 GCNConv layers, each with 256 hidden nodes, plus one HaarPooling layer followed by 2 fully connected layers, each with 256 hidden nodes, and by one fully connected layer with 128 hidden nodes. Table~\ref{tab:architecture} shows the GNN architecture for each dataset.

\begin{table*}[t]
\centering
\begin{minipage}{0.7\textwidth}
\centering
	\caption{Performance comparison for graph classification tasks
(test accuracy in percent, showing the standard deviation over ten repetitions of the experiment).}\label{tab1}
\end{minipage}
\begin{center}
\begin{small}
\begin{threeparttable}
\begin{tabularx}{380pt}{l *6{>{\Centering}X}}
\toprule
\newcommand{\nz}{\phantom{*}}
{\bf Method} &  MUTAG & PROTEINS & NCI1 & NCI109 & MUTAGEN  \\
\midrule
CSM         & 85.4 & --   & --    &-- & -- \\
GIN         & 89.4 & 76.2  & \textbf{82.7}   &-- & -- \\
SortPool    & 85.8 & 75.5  & 74.4 & 72.3* & 78.8* \\
DiffPool   &-- &76.3   & 76.0*   &74.1* & 80.6* \\
gPool      & -- & 77.7 &--   &-- & -- \\
SAGPool   & -- & 72.1  & 74.2              & 74.1 & -- \\
EigenPool & -- & 76.6  & 77.0   & 74.9 & 79.5 \\
\midrule
HaarPool (ours) & \textbf{90.0$\pm$3.6} & \textbf{80.4$\pm$1.8} & 78.6$\pm$0.5   & \textbf{75.6$\pm$1.2} & \textbf{80.9$\pm$1.5}\\
\bottomrule
\end{tabularx}
\centering
\begin{tablenotes}
        \item[] {\rm `*' indicates records retrieved from EigenPool \citep{ma2019graph}, `--' means that there are no public records for the method on the dataset, and bold font is used to highlight the best performance in the list.}
\end{tablenotes}
\end{threeparttable}
\end{small}
\end{center}
\vskip -0.1in
\end{table*}
\paragraph{Results}
Table~\ref{tab1} reports the classification test accuracy. GNNs with HaarPooling have excellent performance on all datasets. In 4 out of 5 datasets, it achieves top accuracy. It shows that HaarPooling, with an appropriate graph convolution, can achieve top performance on a variety of graph classification tasks, and in some cases, improve state of the art by a few percentage points.
\vspace{-4mm}
\begin{table}[htbp!]
\caption{Network architecture.}\label{tab:architecture}
\vskip 0.05in
\begin{center}
\begin{small}
\begin{tabularx}{220pt}{l *2{>{\Centering}X}}
\toprule
{\bf Dataset}  & \textbf{Layers and \#Hidden Nodes}\\
\midrule
MUTAG &GC60-HP-FC60-FC180-FC60\\
PROTEINS  &2GC128-HP-2GC128-HP-2GC128-\\
& HP-GC128-2FC128-FC64 \\
NCI1  &2GC256-HP-FC256-FC1024-FC2048 \\
NCI109 &3GC256-HP-2FC256-FC128\\
MUTAGEN   &3GC256-HP-2FC256-FC128\\
\bottomrule
\end{tabularx}
\end{small}
\end{center}
\vskip -0.1in
\end{table}

{\vspace{3mm}
\subsection{HaarPooling on Triangles Classification}}
We test GNN with HaarPooling on the graph dataset \textbf{Triangles} \citep{knyazev2019understanding}.
\textbf{Triangles} is a 10 class classification problem with 45,000 graphs. The average numbers of nodes and edges of the graphs are 20.85 and 32.74, respectively. In the experiment, the network utilizes GIN convolution \citep{GIN} as graph convolution and either HaarPooling or SAGPooling \citep{lee2019self}. For SAGPooling, the network applies two combined layers of GIN convolution and SAGPooling, which is followed by the combined layers of GIN convolution and \emph{global max pooling}. We write its architecture as GIN-SP-GIN-SP-GIN-MP, where SP means the SAGPooling and MP is the global max pooling. For HaarPooling, we examine two architectures: GIN-HP-GIN-HP-GIN-MP and GIN-HP-GIN-GIN-MP, where HP stands for HaarPooling. We split the data into training, validation, and test sets of size 35,000, 5,000, and 10,000. The number of nodes in the convolutional layers are all set to 64; the batch size is 60; the learning rate is 0.001.

Table~\ref{tab:triangles} shows the training, validation, and test accuracy of the three networks. It shows that both networks with HaarPooling outperform that with SAGPooling.
\begin{table}[htbp!]
	\caption{Results on the Triangles dataset.}\label{tab:triangles}
\begin{center}
\begin{small}
\begin{tabularx}{220pt}{l *3{>{\Centering}X}}\toprule
\multirow{2}{*}{\bf Architecture} & \multicolumn{3}{c}{Accuracy (\%)}\\ \cline{2-4}
&Train & Val & Test\\
\midrule
GIN-SP-GIN-SP-GIN-MP & 45.6 & 45.3 &44.0\\
GIN-HP-GIN-HP-GIN-MP (ours) & 47.5 & 46.3 & 46.1\\
GIN-HP-GIN-GIN-MP (ours)  & 47.3 & 45.8 & 45.5\\
\bottomrule
\end{tabularx}
\end{small}
\end{center}
\vskip -0.2in
\end{table}

{\vspace{2mm}
\subsection{HaarPooling for Quantum Chemistry Regression}}
\paragraph{QM7} In this part, we test the performance of the GNN model equipped with the HaarPooling layer on the QM7 dataset.  People have recently used the QM7 to measure the efficacy of machine-learning methods for quantum chemistry \cite{BlRe2009,RuTkMuLi2012}. The QM7 dataset contains 7,165 molecules, each of which is represented by the Coulomb (energy) matrix and labeled with the atomization energy. Each molecule contains up to 23 atoms. We treat each molecule as a weighted graph: atoms as nodes and the Coulomb matrix of the molecule as the adjacency matrix. Since the node (atom) itself does not have feature information, we set the node feature to a constant vector (i.e., the vector with components all 1), so that features here are uninformative, and only the molecule structure is concerned in learning. The task is to predict the atomization energy value of each molecule graph, which boils down to a standard graph regression problem.

\paragraph{Methods in comparison} We use the same GNN architecture to test HaarPool and SAGPool \cite{lee2019self}: one GCN layer, one graph pooling layer, plus one 3-layer MLP. We compare the performance (test MAE) of the GCN-HaarPool against the GCN-SAGPool and other methods including Random Forest (RF) \citep{Breiman2001}, Multitask Networks (Multitask) \citep{Ramsundar_etal2015}, Kernel Ridge Regression (KRR) \citep{CoVa1995}, Graph Convolutional models (GC) \citep{AlRaPaPa2017}.
\vspace{-3mm}
\begin{table}[ht]
\caption{Test mean absolute error (MAE) comparison on QM7, with the standard deviation over ten repetitions of the experiments.}\label{tab:qm7_results}
\begin{center}
\begin{small}
\begin{tabular}{c|c}
\toprule
Method & Test MAE  \\
\midrule
  RF & $122.7\pm4.2$ \\
  Multitask & $123.7\pm15.6$ \\
  KRR  & $110.3\pm4.7$ \\
  GC  & $77.9\pm2.1$ \\
\midrule
GCN-SAGPool  &43.3 $\pm1.6$ \\
\midrule
GCN-HaarPool (ours)  &\textbf{42.9 $\pm$ 1.2}\\
\bottomrule
\end{tabular}
\end{small}
\end{center}
\vskip -0.1in
\end{table}

\paragraph{Experimental setting} In the experiment, we normalize the label value by subtracting the mean and scaling the standard deviation (Std Dev) to 1. We then need to convert the predicted output to the original label domain (by re-scaling and adding the mean back). Following \citet{Gilmer_etal2017}, we use mean squared error (MSE) as the loss for training and mean absolute error (MAE) as the evaluation metric for validation and test.
Similar to the graph classification tasks studied above, we use PyTorch Geometric \citep{fey2019fast} to implement the models of GCN-HaarPool and GCN-SAGPool, and run the experiment under the GPU computing environment in the Google Cloud AI Platform. Here, the splitting percentages for training, validation, and test are 80\%, 10\%, and 10\%, respectively. We set the hidden dimension of the GCN layer as 64, the Adam for optimization with the learning rate 5.0e-4, and the maximal epoch 50 with no early stop.  We do not use dropout as it would slightly lower the performance. For better comparison, we repeat all experiments ten times with different random seeds.

Table~\ref{tab:qm7_results} shows the results for GCN-HaarPool and GCN-SAGPool, together with the public results of the other methods from \citet{Wu_etal2018}. Compared to the GCN-SAGPool, the GCN-HaarPool has a lower average test MAE and a smaller Std Dev and ranks the top in the table.
Given the simple architecture of our GCN-HaarPool model, we can interpret the GCN-HaarPool an effective method, although its prediction result does not rank the top in Table~9 reported in \citet{Wu_etal2018}.
To further demonstrate that HaarPooling can benefit graph representation learning, we present in Figure~\ref{fig:visualization_qm7} the mean and Std Dev of the training MSE loss (for normalized input) and the validation MAE (which is in the original label domain) versus the epoch. It illustrates that the learning and generalization capabilities of the GCN-HaarPool are better than those of the GCN-SAGPool; in this aspect, HaarPooling provides a more efficient graph pooling for GNN in this graph regression task.

\begin{figure}[t]
\begin{center}
\hspace{-2mm}\includegraphics[width=0.465\textwidth,height=4.5cm]{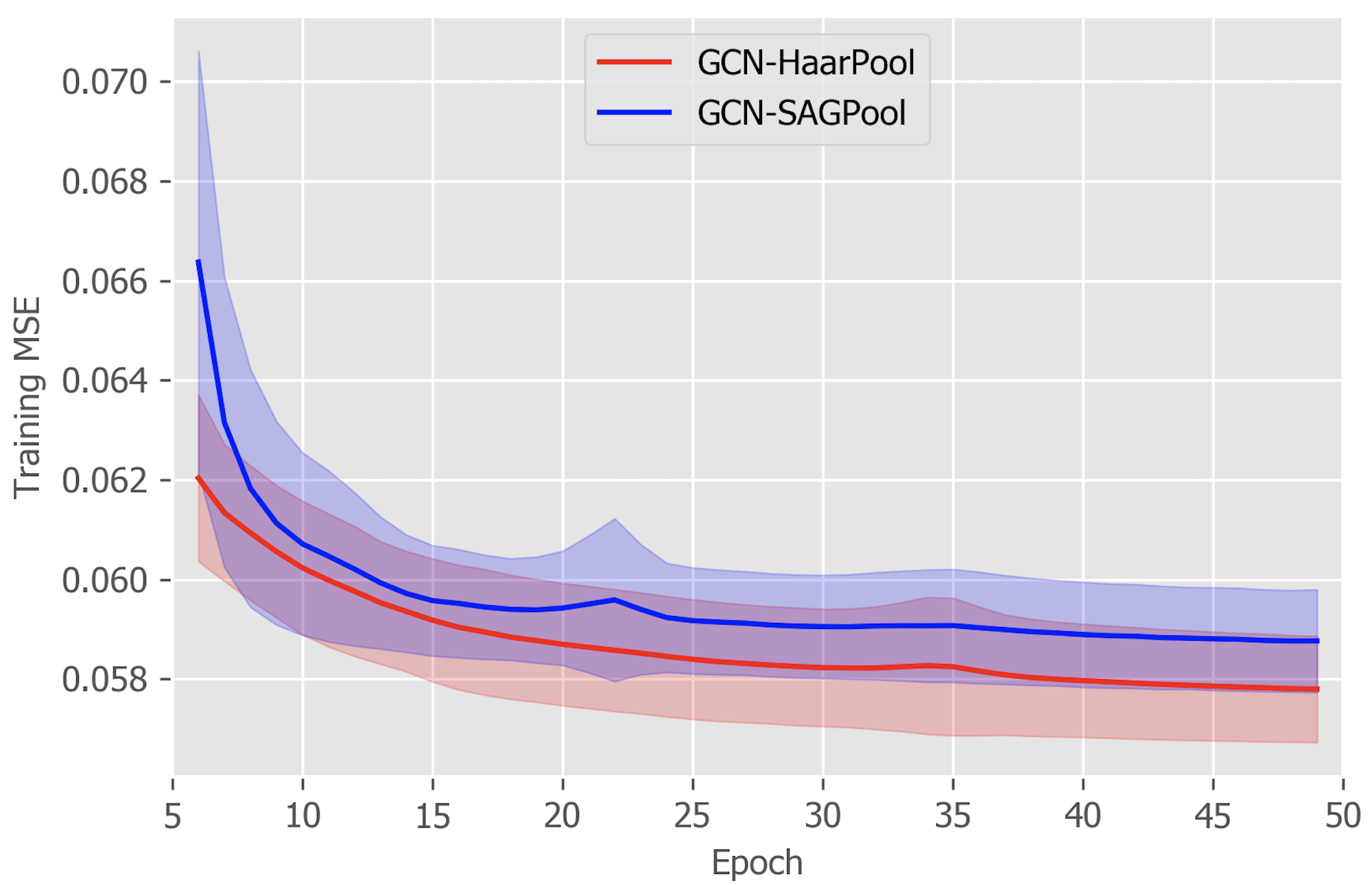}\vskip 1mm
\includegraphics[width=0.45\textwidth,height=4.5cm]{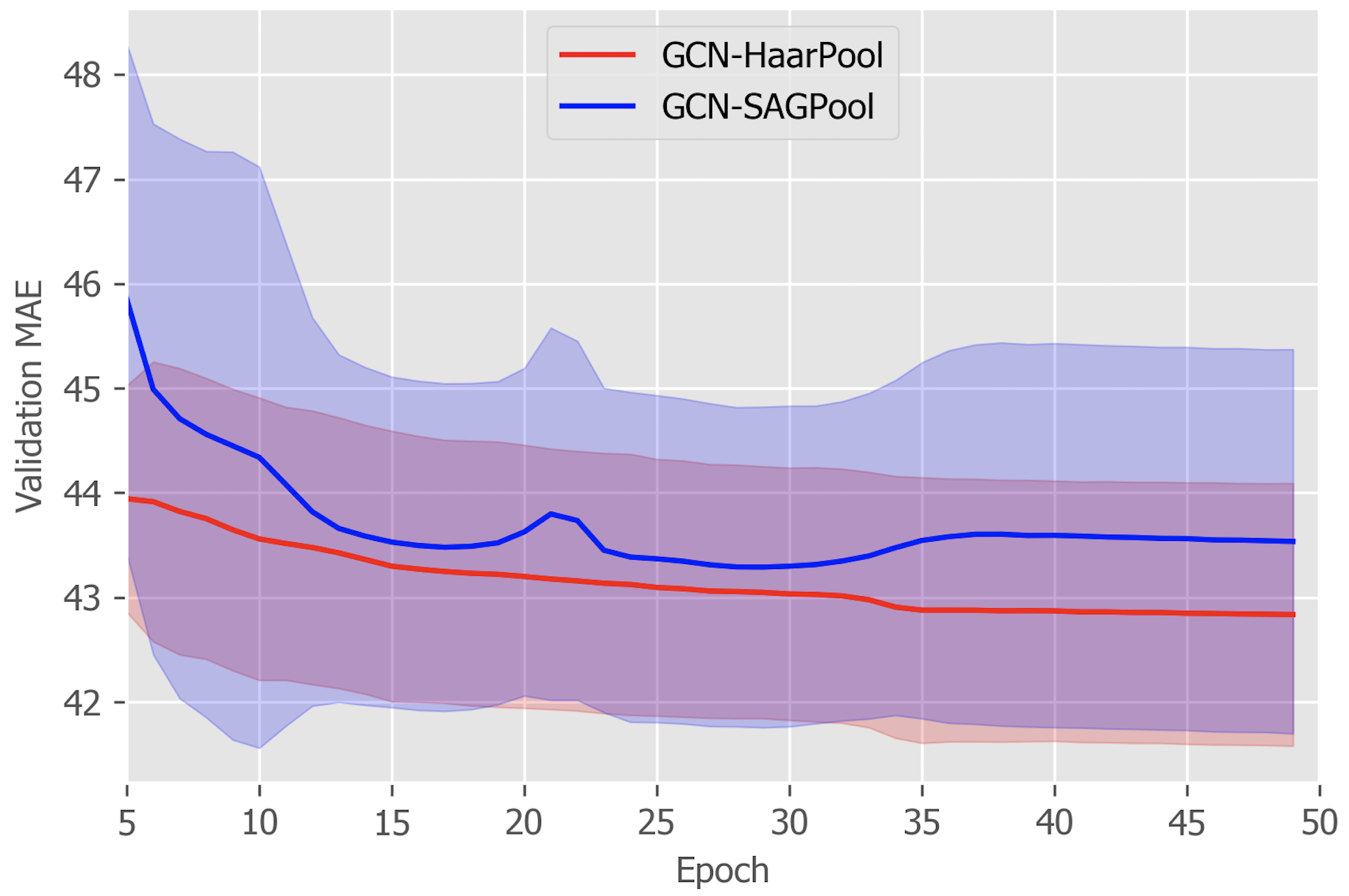}
\vskip -1mm
\caption{Visualization of the training MSE loss (top) and validation MAE (bottom) for GCN-HaarPool and GCN-SAGPool.}\label{fig:visualization_qm7}
\end{center}
\vskip -0.202in
\end{figure}

\section{Computational Complexity}
In the supplementary material, we show the time complexity comparison of HaarPooling and other existing pooling methods. HaarPool is the only algorithm in this table which has near-linear time complexity to the node number. HaarPooling can be even faster in practice, as the cost of the compressive Haar transform is dependent on the sparsity of the Haar basis matrix. The sparsity of the compressive Haar basis matrix is mainly reliant on the chain/tree for the graph. From our construction, the compressive Haar basis matrix is always highly sparse. Thus, the computational cost does not skyrocket as the size of the graph increases.

For empirical comparison, we computed the GPU time for HaarPool and TopKPool on a sequence of datasets of random graphs, as shown in Figure~\ref{fig:gpu_time_Haar_TopK}. For each run, we fix the number of edges of the graphs.
For different runs, the number of the edges ranges from 4,000 to 121,000. The sparsity of the adjacency matrix of the random graph is set to 10\%. The following table shows the average GPU time (in seconds) for pooling a minibatch of 50 graphs. For both pooling methods, we use the same network architecture and one pooling layer, and same network hyperparameters, and run under the same GPU computing environment. 

Figure~\ref{fig:gpu_time_Haar_TopK} shows that the cost of HaarPool does not change much as the edge number increases, while the cost of TopKPool increases rapidly. When the edge number is at most 25000, TopKPool runs slightly faster than HaarPool, but when the number exceeds 25000, the GPU time of TopKPool is longer.

\begin{figure}
    \centering
    \includegraphics[width=\columnwidth]{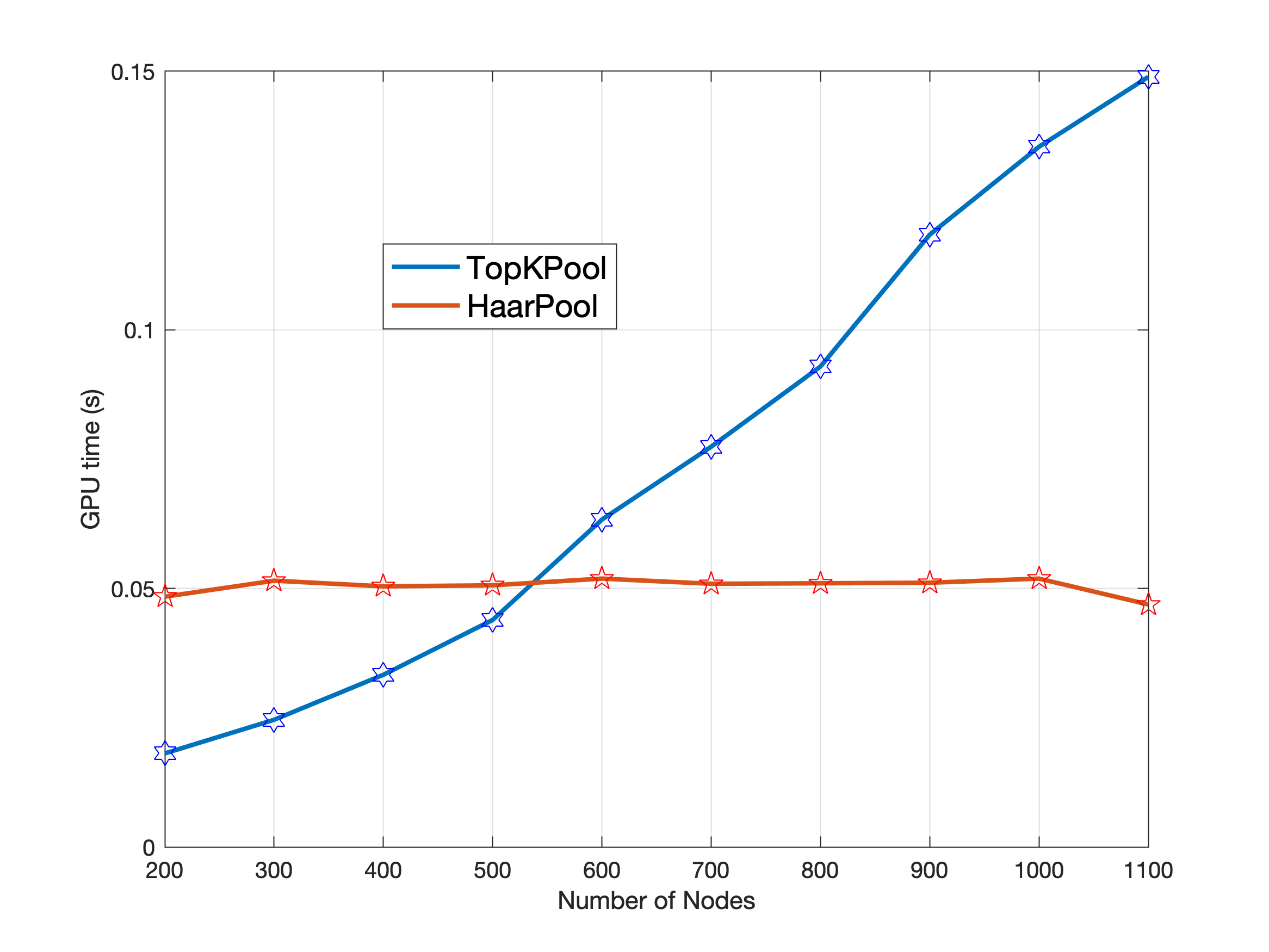}\vspace{-7mm}
    \caption{GPU time comparison of HaarPool with TopKPool for random graphs with up to 1,100 nodes; the Y-axis is the mean GPU time(in seconds) for pooling a minibatch of 50 graphs.}
    \label{fig:gpu_time_Haar_TopK}
\end{figure}

The computational cost of clustered tree generation depends on the clustering algorithms one uses. In the paper, we take spectral clustering as a typical example. Spectral clustering has good performance on various datasets, and its time complexity, though not linear, is smaller than the $k$-means clustering. The METIS is also a good candidate that has a fast implementation. As the Haar basis generation can be pre-computed, the time complexity of clustering has no impact on the complexity of pooling.

By our test, the number of layers of the chain for HaarPool has a substantial impact on the performance of GNNs, and the randomness in the clustering algorithm has little effect on the stability of GNNs.

\section{Conclusion}
\label{sec:conclusion}
We introduced a new graph pooling method called HaarPooling. It has a mathematical formalism derived from compressive Haar transforms. Unlike existing graph pooling methods, HaarPooling takes into account both the graph structure and the features over the nodes, to compute a coarsened representation. The implementation of HaarPooling is simple as the Haar basis and its transforms can be computed directly by the explicit formula. The time and space complexities of HaarPooling are cheap, $\mathcal{O}(|V|)$ and $\mathcal{O}(|V|^2\epsilon)$ for sparsity $\epsilon$ of Haar basis, respectively. As an individual unit, HaarPooling can be applied in conjunction with any graph convolution in GNNs.
We show in experiments that HaarPooling reaches and in several cases surpasses state of the art performance in multiple
graph classification and regression tasks.

\subsubsection*{Acknowledgments}
Ming Li acknowledges support from the National Natural Science Foundation of China (No.\ 61802132 and 61877020). Yu Guang Wang acknowledges support from the Australian Research Council under Discovery Project DP180100506. 
Guido Mont\'{u}far has received funding from the European Research Council (ERC) under the European Union's Horizon 2020 research and innovation programme (grant agreement n\textsuperscript{o} 757983).
Xiaosheng Zhuang acknowledges support in part from Research Grants Council of Hong Kong (Project No. CityU 11301419). 
This material is based upon work supported by the National Science Foundation under Grant No.~DMS-1439786 while Zheng Ma, Guido Mont\'{u}far and Yu Guang Wang were in residence at the Institute for Computational and Experimental Research in Mathematics in Providence, RI, during Collaborate@ICERM on ``Geometry of Data and Networks''. 
Part of this research was performed while Guido Mont\'{u}far and Yu Guang Wang were at the Institute for Pure and Applied Mathematics (IPAM), which is supported by the National Science Foundation (Grant No.~DMS-1440415). 

\bibliographystyle{icml2020}
\bibliography{haarpool_icml}

\newpage
\appendix
\section{Efficient Computation for HaarPooling}
For the HaarPooling introduced in Definition~1, we can develop a fast computational strategy by virtue of fast adjoint Haar transforms.
Let $\gph_{0\to K}$ be a coarse-grained chain of the graph $\gph_0$. For convenience, we label the vertices of the level-$j$ graph $\gph_j$ by $V_j:=\bigl\{\vj_1,\ldots,\vj_{N_j}\bigr\}$.

\paragraph{An efficient algorithm for HaarPooling}
The HaarPooling can be computed efficiently by using the hierarchical structure of the chain, as we introduce as follows. For $j=1,\dots,K$, let $c_{k}^{(j)}$ be the number of children of $\vj_k$, i.e. the number of vertices of $\gph_{j-1}$ which belongs to the cluster $\vj_k$, for $k=1,\ldots,N_j$. For $j=0$, we let $c_{k}^{(0)}\equiv1$ for $k=1,\dots,N_0$. Now, for $j=0,\dots,K$ and $k=1,\dots,N_j$, define the weight for the node $\vj_k$ of layer $j$ by
\begin{equation}\label{eq:wjk}
	\wj_k:=\frac{1}{\sqrt{c_{k}^{(j)}}}.
\end{equation}
Let $W_{0\to K}:=\{\wj_k\, | \, j=0,\dots,K,\: k=1,\dots,N_j\}$.
Then, for $j=0,\dots,K$, the weighted chain $(\gph_{j\to K},W_{j\to K})$ becomes a \emph{filtration} if each parent of the chain $\gph_{j\to K}$ has at least two children.

Let $j=0,\dots,K$. For the $j$th HaarPooling layer,
let $\{\phi_{\ell}^{(j)}\}_{\ell=1}^{N_j}$ be the Haar basis for the $j$th layer, which we also call the Haar basis for the filtration $(\gph_{j\to K},W_{j\to K})$ of a graph $\gph$.
For $k=1,\dots,N_j$, we let $X(v_k^{(j)})=X(v_k^{(j)},\cdot)\in \R^{d_j}$ the feature vector at node $v_k^{(j)}$.
We define the weighted sum for feature $X\in \R^{N_j\times d_j}$ for $d_j\geq1$ by
\begin{equation}\label{eq:ws1}
	\ws[j]\bigl(X,\vj[j]_k\bigr):= X(\vj[j]_k), \quad \vj[j]_k\in \gph_j,
\end{equation}
and recursively, for $i=j+1,\dots,K$ and $\vj[i]_k\in \gph_{i}$,
\begin{equation}\label{eq:wsj}
	\ws[i]\bigl(X,\vj[i]_k\bigr):= \sum_{\vj[i-1]_{k'}\in \vj[i]_k}\wj[i-1]_{k'} \ws[i-1]\bigl(X,\vj[i-1]_{k'}\bigr).
\end{equation}
For each vertex $\vj[i]_k$ of $\gph_i$, the $\ws[i]\bigl(X,\vj[i]_k\bigr)$ is the weighted sum of the $\ws[i-1]\bigl(X,\vj[i-1]_{k'}\bigr)$ at the level $i-1$ for those vertices $\vj[i-1]_{k'}$ of $\gph_{i-1}$ whose parent is $\vj[i]_k$.

\vspace{4mm}
\begin{theorem}\label{thm:fasthaarpool}
For $0\leq j\leq K-1$, let $\{\eigfm^{(i)}\}_{\ell=1}^{N_i}$ for $i=j+1,\dots,K$ be the Haar bases for the filtration $(\gph_{j\to K},W_{j\to K})$ at layer $i$. Then, the compressive Haar transform for the $j$th HaarPooling layer can be computed by,
for the feature $X\in \R^{N_j\times d_j}$ and $\ell=1,\dots,N_j$,
\begin{equation}\label{eq:adftbyws}
	\left(\Phi_j^T X\right)_{\ell} = \sum_{k=1}^{N_i} \ws[i]\bigl(X,\vj[i]_k\bigr)\wj[i]_k \eigfm^{(i)}(\vj[i]_k),
\end{equation}
where $i$ is the largest possible number in $\{j+1,\ldots,K\}$ such that  $\eigfm^{(i)}$ is the $\ell$th member of the orthonormal basis $\{\eigfm^{(i)}\}_{\ell=1}^{N_i}$ for $l_2(\gph_i)$, $\vj[i]_k$ are the vertices of $\gph_i$ and the weights  $\wj[i]_k$ are given by Equation~(\ref{eq:wjk}).
\end{theorem}
We give the algorithmic implementation of Theorem~\ref{thm:fasthaarpool} in Algorithm~\ref{alg:fasthaarpool}, which provides a fast algorithm for HaarPooling at each layer.

\begin{algorithm}[htbp!]
   \caption{Fast HaarPooling for One Layer}
   \label{alg:fasthaarpool}
\begin{algorithmic}
   \STATE {\bfseries Input:} Input feature $X_j^{\rm in}$ for the $j$th pooling layer given $j=0,\dots,K-1$ in a GNN with  total $K$ HaarPooling layers; the chain $\gph_{j\to K}$ associated with the HaarPooling; numbers $N_i$ of nodes for layers $i=j,\dots,K$.
   \vskip 0.06in
   \STATE {\bfseries Output:} $\Phi_{j}^T X_j^{\rm in}$ from Definition~1.\vspace{1mm}
   \STATE{\textbf{Step 1:}}
   Evaluate the sums for $i=j,\ldots,K$ recursively, using Equations~(\ref{eq:ws1}) and (\ref{eq:wsj}):
 \STATE{\quad$\ws[i]\bigl(X_j^{\rm in},\vj[i]_k\bigr)\quad \forall \vj[i]_k \in V_i$ .}
    \STATE{\textbf{Step 2:}}
   \FOR{$\ell=1$ {\bfseries to} $N_{j+1}$} \vspace{1mm}
   \STATE {
   Set $N_{K}=0$.\\
   Compute $i$ such that $N_{i+1}+1\leq \ell\leq N_{i}$.\\
   Evaluate $\sum_{k=1}^{N_i} \ws[i](X_j^{\rm in},\vj[i]_k)\wj[i]_k\eigfm^{(i)}(\vj[i]_k)$
      in Equation~(\ref{eq:adftbyws}) by the two steps:\\
    (a)~Compute the product for all $\vj[i]_k\in V_i$:\\
    $\hspace{3mm}T_\ell(X_j^{\rm in},\vj[i]_k)=\ws[i](X_j^{\rm in},\vj[i]_k)\wj[i]_k\eigfm^{(i)}(\vj[i]_k).$\\
    (b)~Evaluate sum $\sum_{k=1}^{N_i}T_\ell(X_j^{\rm in},\vj[i]_k)$.\\[1mm]
}
   \ENDFOR
\end{algorithmic}
\end{algorithm}

\section{Proofs}\label{appendix:proof}
\begin{proof}[Proof for Equation~(6) in Section~4]
    We only need to prove the first formula. The second is obtained by definition. To simplify notation, we let $f=X_j^{\rm in}$. By construction of Haar basis, for some layer $j$, the first $N_{j+1}$ basis vectors
    \begin{equation*}
        \phi_{\ell}^{(j)}(v)= \phi_{\ell}^{(j+1)}(p)/\sqrt{|\parents_{\gph}(v)|},\quad \mbox{for~}p=\parents_{\gph}(v).
    \end{equation*}
    Then, the Fourier coefficient of $f$ for the $\ell$th basis vector is the inner product
    \begin{align*}
        \left\langle f,\phi_{\ell}^{(j)}\right\rangle
        &= \sum_{v\in \gph_j}f(v)\overline{\phi_{\ell}^{(j)}(v)}\\
        &= \sum_{p\in \gph_{j+1}}\sum_{p=\parents_{\gph}(v)}f(v)\overline{\phi_{\ell}^{(j+1)}(p)}/\sqrt{|\parents_{\gph}(v)|}\\
        &= \sum_{p\in \gph_{j+1}}\widetilde{f}(p)\overline{\phi_{\ell}^{(j+1)}(p)}        = \left\langle\widetilde{f}, \phi_{\ell}^{(j+1)}\right\rangle
    \end{align*}
    where we have let
    \begin{equation*}
        \widetilde{f}(p) := \frac{1}{\sqrt{|\parents_{\gph}(v)|}}\sum_{p=\parents_{\gph}(v)}f(v).
    \end{equation*}
    This then gives
    \begin{equation}\label{eq:tr_fouriercoeff}
        \sum_{\ell=1}^{N_{j+1}}\left|\left\langle f,\phi_{\ell}^{(j)}\right\rangle\right|^2
        = \sum_{\ell=1}^{N_{j+1}}\left|\left\langle\widetilde{f}, \phi_{\ell}^{(j+1)}\right\rangle\right|^2.
    \end{equation}
    Since $\{\phi_{\ell}\}_{\ell=1}^{N_{j+1}}$ forms an orthonormal basis for $\ell_2(\gph_{j+1})$,
    \begin{align*}
       \bigl\|\Phi_j^T f\bigr\|^2
       &= \sum_{\ell=1}^{N_{j+1}}\left|\left\langle\widetilde{f}, \phi_{\ell}^{(j+1)}\right\rangle\right|^2
        = \bigl\|\widetilde{f}\bigr\|^2= \sum_{p\in\gph_{j+1}}\bigl|\widetilde{f}(p)\bigr|^2\\
        &=\sum_{p\in \gph_{j+1}}\left|\frac{1}{\sqrt{|\parents_\gph(v)|}}\sum_{p=\parents_\gph(v)}f(v)\right|^2.
    \end{align*}
    This proves the left formula in Equation~(6) in Section~4.
\end{proof}

\begin{proof}[Proof of Theorem~\ref{thm:fasthaarpool}]
	By the relation between $\eigfm^{(i)}$ and $\eigfm^{(j)}$, for $i=j+1,\dots,K$ and $\ell=1,\dots,N_{j+1}$,
	\begin{align*}
		&\left(\Phi_j^T X\right)_{\ell}= \sum_{k=1}^{N_j}X(\vj[j]_k)\eigfm^{(j)}(\vj[j]_k)\\
		&= \sum_{k'=1}^{N_{j+1}} \left(\sum_{\parents_\gph(\vj[j]_k)=\vj[j+1]_{k'}}X(\vj[j]_k)\right)\wj[j+1]_{k'}\eigfm^{(j+1)}(\vj[j+1]_{k'})\\
		&=\sum_{k'=1}^{N_{j+1}} \ws[j+1](X,\vj[j+1]_{k'})\wj[j+1]_{k'}\eigfm^{(j+1)}(\vj[j+1]_{k'})\\
		&=\sum_{k''=1}^{N_{j+2}} \left(\sum_{\parents_\gph(\vj[j+1]_{k'})=\vj[j+2]_{k''}}\ws[j+1](X,\vj[j+1]_{k'})\wj[j+1]_{k'}\right)\\
		&\quad\times\wj[j+2]_{k''}\eigfm^{(j+2)}(\vj[j+2]_{k''})\\
		&=\sum_{k''=1}^{N_{j+2}} \ws[j+2](X,\vj[j+2]_{k''})\wj[j+2]_{k''}\eigfm^{(j+2)}(\vj[j+2]_{k''})\\
		&\cdots\cdots\\
		&=\sum_{k=1}^{N_i} \ws[i](X,\vj[i]_k)\wj[i]_k\eigfm^{(i)}(\vj[i]_k),
	\end{align*}
where $v_{k'}^{(j+1)}$ is the parent of $v_k^{(j)}$ and $v_{k''}^{(j+2)}$ is the parent of $v_k^{(j+1)}$, and we recursively compute the summation to obtain the last equality,
thus completing the proof.
\end{proof}

\section{Experimental Setting}
The hyperparameters include batch size; learning rate, weight decay rate (these two for optimization); the maximal number of epochs; patience for early stopping. Table~\ref{tab:hyperparameter} shows the choice of hyperparameters for classification benchmark datasets.
{\footnotesize
\begin{table*}[ht!]
\caption{Hyperparameter setting}\label{tab:hyperparameter}
\vspace{2mm}
\begin{center}
{\tabcolsep=0pt\def\arraystretch{1.2}
\begin{tabularx}{380pt}{l *7{>{\Centering}X}}
\toprule
{\bf Data Set}  & MUTAG & PROTEINS & NCI1 & NCI109 & MUTAGEN\\
\midrule
batch size &60 & 50  & 100 & 100 & 100\\
max \#epochs  &30 &20  &150 & 150 & 50 \\
early stopping  &15 &20  & 50 & 50 & 50 \\
learning rate &0.01 &0.001 &0.001 &0.01 &0.01\\
weight decay    &0.0005 & 0.0005  &0.0005& 0.0001& 0.0005\\
\bottomrule
\end{tabularx}}
\end{center}
\end{table*}}
\newcolumntype{b}{>{\hsize=1\hsize}X}
\newcolumntype{s}{>{\hsize=.7\hsize}X}
\begin{table*}[t]
\small
\caption{Property comparison for pooling methods}\label{tab:comparison_pooling}
\vspace{2mm}
\begin{center}
\begin{threeparttable}
{\tabcolsep=0pt\def\arraystretch{1.2}
\begin{tabularx}{\textwidth}{l *3{>{\Centering}b} *1{>{\Centering}s}*1{>{\Centering}b}*3{>{\Centering}s}}
\toprule
\newcommand{\nz}{\phantom{*}}
{\bf Method} &  Time Complexity & Space Complexity &  Clustering-based & Spectral-based & Hierarchical Pooling &  Use Node Feature & Use Graph Structure & Sparse Representation\\
\midrule
SortPool    & $\mathcal{O}(|V|^2)$ & $\mathcal{O}(|V|)$ &    &  &  & \checkmark & & \\
DiffPool   & $\mathcal{O}(|V|^2)$ & $\mathcal{O}(k|V|^2)$ &    &  & \checkmark & \checkmark &  & \\
gPool      & $\mathcal{O}(|V|^2)$ & $\mathcal{O}(|V|+|E|)$ &   &  & \checkmark & \checkmark & &  \\
SAGPool   & $\mathcal{O}(|E|)$ & $\mathcal{O}(|V|+|E|)$ &    &  & \checkmark  & \checkmark & \checkmark & \\
EigenPool & $\mathcal{O}(|V|^2)$ & $\mathcal{O}(|V|^2)$ &  \checkmark  & \checkmark & \checkmark & \checkmark & \checkmark & \\
\midrule
HaarPool  & $\mathcal{O}(|V|)$ & $\mathcal{O}(|V|^2\epsilon)$ &  \checkmark  & \checkmark & \checkmark  & \checkmark & \checkmark & \checkmark\\
\bottomrule
\end{tabularx}}
\begin{tablenotes}
        \item[] `$|V|$' is the number of vertices of the input graph; `$|E|$' is the number of edges of the input graph; `$\epsilon$' in HaarPooling is the sparsity of the compressive Haar transform matrix; `$k$' in the DiffPool is the pooling ratio.
\end{tablenotes}
\end{threeparttable}
\end{center}
\end{table*}

\section{Property Comparison of Pooling Methods}
Here we provide a comparison of the properties of HaarPooling
with existing pooling methods. The properties in the
comparison include time complexity and space complexity,
and whether involving the clustering, hierarchical pooling
(which is then not a global pooling), spectral-based, node
feature or graph structure, and sparse representation. We
compare HaarPooling (denoted by HaarPool in the table) to
other methods (SortPool, DiffPool, gPool, SAGPool, and
EigenPool).

\begin{itemize}
  \item The SortPool (i.e., SortPooling) is a global pooling which uses node signature (i.e., Weisfeiler-Lehman color of vertex) to sort all vertices by the values of the channels of the input data. Thus, the time complexity (worst case) of SortPool is $\mathcal{O}(|V|^2)$ and space complexity is $\mathcal{O}(|V|)$.
Other pooling methods mentioned here are all hierarchical pooling.
  \item DiffPool and gPool both use the node feature and have time complexity $\mathcal{O}(|V|^2)$. The DiffPool learns the assignment matrices in an end-to-end manner and has space complexity $\mathcal{O}(k|V|^2)$ for pooling ratio $k$. The gPool projects all nodes to a learnable vector to generate scores for nodes, and then sorts the nodes by the projection scores; the space complexity is $\mathcal{O}(|V|+|E|)$.
  \item SAGPool uses the graph convolution to calculate the attention scores of nodes and then selects top-ranked nodes for pooling. The time complexity of SAGPool is $\mathcal{O}(|E|)$, and the space complexity is $\mathcal{O}(|V|+|E|)$ due to the sparsity of the pooling matrix.
  \item EigenPool, which considers both the node feature and graph structure, uses the eigendecomposition of subgraphs (from clustering) of the input graph, and pools the input data by the Fourier transforms of the assembled basis matrix. Due to the computational cost of eigendecomposition, the time complexity of EigenPool is $\mathcal{O}(|V|^2)$, and the space complexity is $\mathcal{O}(|V|^2)$.
  \item HaarPool which uses the sparse representation of data by compressive Haar basis has linear time complexity $\mathcal{O}(|V|)$ (up to a $\log |V|$ term), and the space complexity is  $\mathcal{O}(|V|^2\epsilon)$, where $\epsilon$ is the sparsity of compressive Haar transform matrix and is usually very small. HaarPooling can be even faster in practice, as the cost of the compressive Haar transform is dependent on the sparsity of the Haar basis matrix. The sparsity of the compressive Haar basis matrix is mainly reliant on the chain/tree for the graph. From our construction, the compressive Haar basis matrix is always highly sparse. Thus, the computational cost does not skyrocket as the size of the graph increases.
\end{itemize}
In Table \ref{tab:comparison_pooling}, the HaarPool is the only pooling method which has time complexity proportional to the number of nodes and thus has a faster implementation.

\end{document}

%% file: math_commands.tex
\usepackage{amsmath,amsfonts,bm,amsthm,amssymb}

\newcommand{\gph}{\mathcal{G}} 

\newcommand{\eigfm}[1][\ell]{\phi_{#1}}
\newcommand{\ws}[1][j]{\mathcal{S}^{(#1)}} 
\newcommand{\vj}[1][j]{v^{(#1)}} 
\newcommand{\wj}[1][j]{w^{(#1)}} 
\newtheorem{theorem}{Theorem}
\newtheorem{proposition}[theorem]{Proposition}
\newtheorem{definition}[theorem]{Definition}

\def\eqref#1{equation~\ref{#1}}

\def\1{\bm{1}}

\DeclareMathAlphabet{\mathsfit}{\encodingdefault}{\sfdefault}{m}{sl}
\SetMathAlphabet{\mathsfit}{bold}{\encodingdefault}{\sfdefault}{bx}{n}

\newcommand{\R}{\mathbb{R}}

\newcommand{\parents}{Pa} 